\documentclass[letterpaper, 10 pt, journal, twoside]{IEEEtran}
\usepackage{amsmath,amsfonts}
\usepackage{algorithmic}
\usepackage{algorithm}
\usepackage{array}
\usepackage{subcaption}
\usepackage{textcomp}
\usepackage{stfloats}
\usepackage{url}
\usepackage{verbatim}
\usepackage{graphicx}
\usepackage{cite}

\usepackage{epstopdf}
\usepackage{float}
\usepackage{setspace}

\usepackage{mathrsfs}
\graphicspath{ {./images/} }
\usepackage{amssymb}
\usepackage{caption}
\usepackage{multirow}
\usepackage{inputenc}
\let\labelindent\relax
\usepackage[inline]{enumitem}
\usepackage{soul}
\usepackage{booktabs}
\usepackage{varwidth}
\usepackage{bm}

\usepackage{tikz}

\usepackage{amsthm}

\DeclareCaptionLabelSeparator{periodspace}{.\quad}
\definecolor{grey1}{RGB}{192,192,192}
\definecolor{grey2}{RGB}{178,178,178}
\definecolor{grey3}{RGB}{150,150,150}
\definecolor{grey4}{RGB}{119,119,119}
\definecolor{grey5}{RGB}{77,77,77}
\definecolor{green}{RGB}{112,173,71}
\definecolor{blue2}{RGB}{68,115,196}
\definecolor{red}{RGB}{192,0,0}
\definecolor{yellow}{RGB}{255,192,0}

\newtheorem{theorem}{{Theorem}}

\newtheorem{remark}{{Remark}}

\setlist*[enumerate,1]{%
  label=(\roman*),
  font=\itshape,
}

\usepackage{multicol}
\usepackage[bookmarks=true]{hyperref}

\newcommand{\reducefiggap}{\vspace*{-1.6em}}
\newcommand{\Ga}{\mathcal{G}_a(x_{a,k})}
\newcommand{\Fa}{\mathcal{F}_a(x_{a,k})}
\newcommand{\Gm}{\mathcal{G}_m(x_{m,k})}
\newcommand{\Fm}{\mathcal{F}_m(x_{m,k})}

\newcommand{\edittext}{\textcolor{black}}

\begin{document}

\title{Bridging the Model-Reality Gap with Lipschitz Network Adaptation}

\author{Siqi Zhou, Karime Pereida, Wenda Zhao, and Angela P. Schoellig
\thanks{Manuscript received: September 9, 2021; Accepted: November 7, 2021.}
\thanks{This paper was recommended for publication by Jens Kober  upon evaluation of the Associate Editor and Reviewers' comments.}
\thanks{This work was supported by the Natural Sciences and Engineering Research Council of Canada (NSERC), the Canada Research Chairs program, and the CIFAR AI Chairs program.
}
\thanks{The authors are with the Dynamic Systems Lab (\href{http://www.dynsyslab.org}{www.dynsyslab.org}), Institute for Aerospace Studies, University of Toronto, Canada and the Vector Institute for Artificial Intelligence, Toronto. Emails: \{siqi.zhou, karime.pereida, wenda.zhao, angela.schoellig\}@robotics.utias.utoronto.ca
}
\thanks{Digital Object Identifier (DOI): see top of this page.}%
}

\markboth{IEEE Robotics and Automation Letters. Preprint Version. Accepted November, 2021}
{Zhou \MakeLowercase{\textit{et al.}}: Bridging the Model-Reality Gap with Lipschitz Network Adaptation} 


\maketitle

\begin{abstract}
As robots venture into the real world, they are subject to unmodeled dynamics and disturbances. Traditional model-based control approaches have been proven successful in relatively static and known operating environments. However, when an accurate model of the robot is not available, model-based design can lead to suboptimal and even unsafe behaviour. In this work, we propose a method that bridges the model-reality gap and enables the application of model-based approaches even if dynamic uncertainties are present. In particular, we present a learning-based model reference adaptation approach that makes a robot system, with possibly uncertain dynamics, behave as a predefined reference model. In turn, the reference model can be used for model-based controller design. In contrast to typical model reference adaptation control approaches, we leverage the representative power of neural networks to capture highly nonlinear dynamics uncertainties and guarantee stability by encoding a certifying Lipschitz condition in the architectural design of a special type of neural network called the Lipschitz network. Our approach applies to a general class of nonlinear control-affine systems even when our prior knowledge about the true robot system is limited. We demonstrate our approach in flying inverted pendulum experiments, where an off-the-shelf quadrotor is challenged to balance an inverted pendulum while hovering or tracking circular trajectories.
\end{abstract}


\section{INTRODUCTION}
\label{sec:introduction}
\IEEEPARstart{A}{dvances} in hardware and algorithms have enabled robots to enter more complex environments and perform increasingly versatile tasks such as home and healthcare services,  search and rescue, aerial package delivery, and industrial inspections. In these applications, robots need to cope with unmodeled dynamics, external time-varying disturbances, and other adverse factors such as communication latency. These practical issues pose challenges to the design of controllers using standard model-based techniques.

In the literature, common model-based control techniques include, but are not limited to, model predictive control (MPC) and linear quadratic regulators (LQR). These approaches are effective when the dynamics model of the robot system is sufficiently accurate and the operating environment does not change significantly over time. When these conditions are not met, model-based designs can lead to suboptimal or unsafe behaviour~\cite{brunke2021safe}. While there exist robust approaches that account for uncertainties by considering worst-case scenarios, these robust techniques can be often overly conservative~\cite{zhou1998essentials}.

An alternative approach to cope with dynamics uncertainty is to enable the system to adapt. One particular set of adaptive approaches is model reference adaptive control (MRAC), which aims to make the controlled system behave similarly to a desired reference model despite unknown disturbances~\cite{sastry2011adaptive}. Although classical adaptive control approaches techniques provide stability guarantees, they usually assume a particular system structure that limit the range of robotic applications to which they can be applied~\cite{sastry2011adaptive}. 

\begin{figure}
    \centering
    \begin{subfigure}{.5\textwidth}
     \includegraphics[width=\columnwidth]{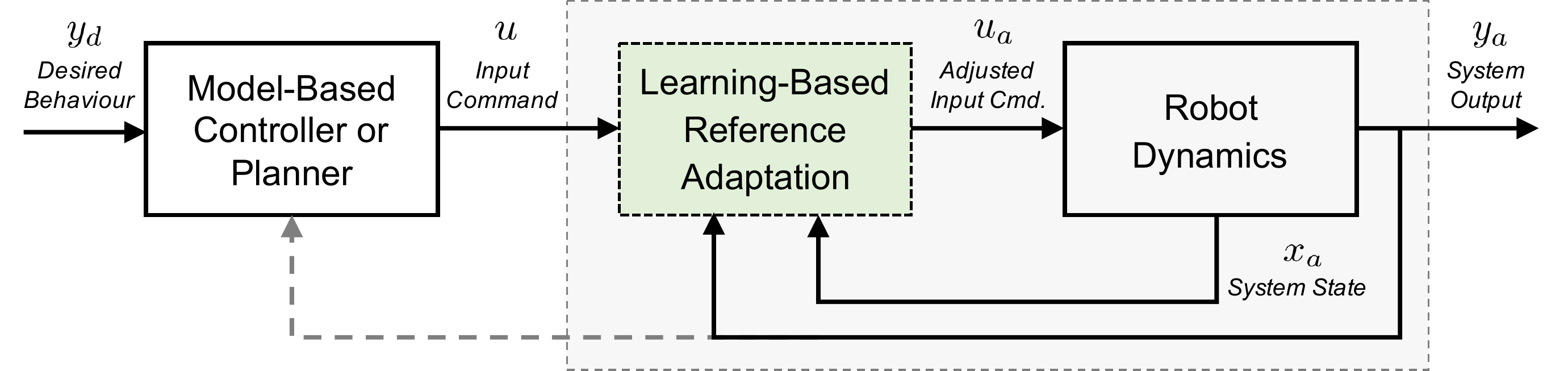}
     \caption{High-level block diagram}
     \label{fig:blockdiagram_highlevel}
     \end{subfigure}
     \begin{subfigure}{.5\textwidth}
     \vspace{1em}
    \includegraphics[width=\columnwidth]{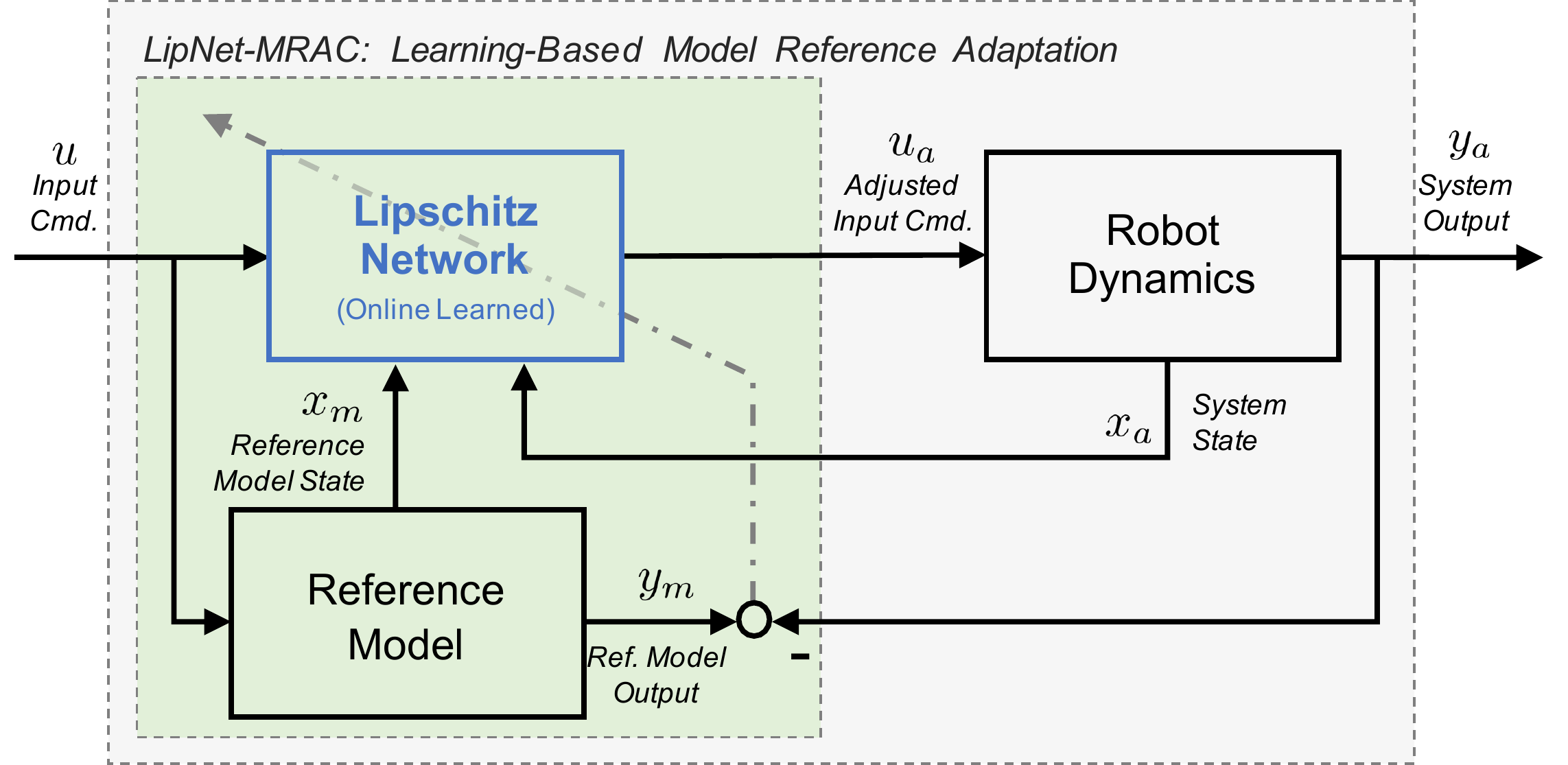}
      \caption{Detailed block diagram of the model reference adaptation module}
      \label{fig:blockdiagram_detailed}
     \end{subfigure}
    \caption{Block diagrams of the proposed approach. The grey box represents the robot system equipped with the proposed learning-based model reference adaptation module (green box).  The reference sent to the robot system is adapted online by a Lipschitz network (blue box) such that the response from the input $u$ to the output $y_a$ resembles the response of a reference model, which can be used in the outer model-based controller or planner. A video of flying inverted pendulum experimental results can be found here: \url{http://tiny.cc/lipnet-pendulum} }
    \label{fig:blockdiagram}
    \reducefiggap
\end{figure}

In this work, 
we propose a novel learning-based MRAC approach that bridges the model-reality gap and enables us to leverage the power and simplicity of model-based control techniques, even in dynamic and uncertain conditions. In particular, we consider the hierarchical architecture illustrated in Figure~\ref{fig:blockdiagram}, where a low-level adaptive module (green box) modifies the input to the system such that the system's input-output response resembles that of the reference model and a high-level controller is designed based on the reference model to achieve a desired robot behaviour. In contrast to existing MRAC approaches such as~\cite{sastry2011adaptive,cooper2014use,chowdhary2014bayesian}, we leverage the expressive power of deep neural networks (DNNs) to capture a broader range of unmodelled dynamics and guarantee stability by exploiting the Lipschitz property of a special type of DNN called Lipschitz network (LipNet)~\cite{anil2019sorting}. 

Our contributions are as follows:
\begin{enumerate}
\item presenting a Lipschitz network adaptive control approach that makes a nonlinear robot system, with possibly unknown dynamics, behave as a reference model,
\item deriving a condition that guarantees stability of the proposed approach by exploiting the Lipschitz properties of the LipNet module, and
\item experimentally verifying the efficacy of the proposed approach for bridging the model-reality gap by balancing an inverted pendulum on a quadrotor platform despite dynamics uncertainties.
\end{enumerate}

\section{RELATED WORK}
\label{sec:related_work}
The rich literature on model-based control approaches shows the effectiveness, safety, and simplicity of these techniques for cases when the dynamics model of system is accurate and the operating environment is static. The model-reality gap is a crucial factor that prevents traditional model-based approaches to be directly applicable to robot systems that are subject to uncertain dynamics and disturbances. One can think of three approaches to address the model-reality gap~\cite{brunke2021safe}:  \textit{(i)} robustness, \textit{(ii)} adaptation, and \textit{(iii)} anticipation.

The \emph{robustness} approach aims to design control laws that are stable for a range of unknown dynamics and disturbances that may affect the robotic system. Robust control approaches include, but are not limited to, sliding-mode control~\cite{chen2016robust}, robust MPC~\cite{limon2010robust}, $H_\infty$ control~\cite{lyu2018disturbance}, as well as more recent domain randomization techniques for sim-to-real transfer~\cite{peng2018sim}. While robust approaches typically guarantee stability and safety in the presence of unmodelled dynamics and disturbances, their performance can be conservative. 

In contrast, \emph{adaptation} approaches address the model-reality gap by adapting or learning online using data collected by the robot. Adaptive controllers such as MRAC~\cite{sastry2011adaptive} and $\mathcal{L}_1$ adaptive control~\cite{Hovakimyan2010} are fast and able to handle unmodelled dynamics. In order to further improve performance, learning-based controllers that leverage past experience are being proposed. Non-parametric approaches include learning-based controllers using Gaussian Processes (GPs)~\cite{ostafew2016learning}, which leverage past experience to learn a better system model, but can be computationally expensive resulting in a slow response to changes in the environment. While there are newer learning MPC approaches using Bayesian linear regression (BLR)~\cite{mckinnon2019}, formal guarantees are not given.

 \emph{Anticipation} approaches address the model-reality gap by learning offline. In DNN-based inverse control, a mapping from desired output to actual output is learned offline~\cite{zhou-ijrr20} and used to improve tracking performance of a quadrotor. In reinforcement learning (RL), latent variables that represent the environment are used to anticipate changes in the environment for off-policy RL~\cite{xie2020deep}. While anticipation approaches are effective for addressing the uncertainties for a broad range of systems, they typically lack the adaptivity to cope with changes during real-time execution.

Adaptive controllers handle unmodelled dynamics and disturbances without the need for conservative control laws or significant amounts of past experience to learn offline. Due to their expensiveness and fixed cost for online inference, neural networks (NNs) are emerging as attractive options for implementing adaptive frameworks on resource-constrained robot platforms. Neural networks have previously been used in online inverse control, but they suffered from a lack of robustness against disturbances~\cite{jordan1992forward} and the need for appropriate initialization in order to converge~\cite{chen1995adaptive}. They have also been used to relax the assumptions of conventional MRAC~(e.g., \cite{Lightbody1995}). However, earlier studies often use radial basis function (RBF) NNs, which require a sufficient preallocation of basis functions over the operating domain; the desired theoretical guarantees do not hold outside of the targeted operating domain. Recently, an asynchronous DNN MRAC framework was proposed to mitigate the limitation of RBF NNs by learning ``features'' at a slower timescale~\cite{joshi2019deep}; but, the approach only considers systems with additive input uncertainties. 

In this work, we consider a more general class of control-affine nonlinear systems and  leverage the expressiveness of DNNs to learn complex dynamic uncertainties. The stability of the adapted system is guaranteed by exploiting the Lipschitz properties of the Lipschitz network. We demonstrate our approach in flying inverted pendulum experiments.

\section{PROBLEM FORMULATION}
\label{sec:problem_formulation}
We consider robot systems whose dynamics can be represented in the following form: 
\begin{equation}
\label{eqn:actual_system}
 \begin{aligned}
    x_{a,k+1} &= f_a(x_{a,k}) + g_a(x_{a,k}) u_{a,k}\\
    y_{a,k} &= h_a(x_{a,k})\,,
 \end{aligned}
\end{equation}
where the subscript $a$ denotes the actual robot system, $k\in\mathbb{Z}_{\geq 0}$ is the discrete-time index, $x_a\in\mathbb{R}^n$ is the system state, $u_a\in\mathbb{R}$ is the system input, $y_a\in\mathbb{R}$ is the system output, and $f_a$, $g_a$, and $h_a$ are smooth nonlinear functions that are possibly unknown a-priori. Our goal is to design a learning-based control law such that the robot behaves as a reference model, which can subsequently be leveraged when designing the outer-loop controller or planner.

The reference model can have the following form:
\begin{equation}
\label{eqn:nonlinear_reference_system}
 \begin{aligned}
    x_{m,k+1} &= f_m(x_{m,k}) + g_m(x_{m,k}) u_{m,k}\\
    y_{m,k} &= h_m(x_{m,k})\,,
 \end{aligned}
\end{equation}
where the subscript $m$ denotes the reference model, the reference model state $x_m$, input $u_m$, and output $y_m$ are defined analogously as in~\eqref{eqn:actual_system}, and $f_m$, $g_m$, and $h_m$ are smooth nonlinear functions. Note that the reference model has a generic control-affine form. Practically, one could use a nonlinear reference model that best captures our prior knowledge about the robot system. Alternatively, to simplify the outer-loop controller design, one may choose a linear reference model
\begin{equation}
\label{eqn:linear_reference_system}
 \begin{aligned}
    x_{m,k+1} &= A_m x_{m,k} + B_m u_{m,k}\\
    y_{m,k} &= C_m x_{m,k}\,,
 \end{aligned}
\end{equation}
where $(A_m, B_m, C_m)$ are constant matrices with consistent dimensions, and use well-established linear control tools.  

We consider a control architecture as shown in Figure~\ref{fig:blockdiagram_detailed}. Without loss of generality, we assume that the inputs to the robot system and the reference model are
\begin{align}
\label{eqn:input_actual}
    u_{a,k} = u_{k} + \delta u_k \text{ and }
    u_{m,k} = u_{k},
\end{align}
where $u_{k}$ is the input command computed by the outer-loop model-based controller. The objective of model reference adaptation is to learn the input adjustment $\delta u_k$ such that the output of the robot system~\eqref{eqn:actual_system}, $y_{a,k}$, tracks the output of the reference system~\eqref{eqn:nonlinear_reference_system}, $y_{m,k}$.

We make the following assumptions: \textit{(i)} the dynamics of the robot system is minimum phase (i.e., has stable forward and inverse dynamics) and has a well-defined relative degree, and \textit{(ii)} the reference model is stable and has the same relative degree as the robot system. As discussed in~\cite{zhou-ijrr20}, the first assumption is necessary to safely apply an inverse dynamics learning approach and is satisfied by closed-loop stabilized robot systems such as quadrotors and manipulators. The second assumption is also not restrictive, as the relative degree of a robot system can be estimated from experiments or inferred from our prior knowledge~\cite{zhou-ijrr20}, and the reference model can be designed to satisfy this assumption. 

\edittext{Note that the choice of a reference model is generally problem dependent. For instance, it can be chosen to achieve a certain desired response or maximize the stability margin of the system. In practice, one would need to choose a reference model that is feasible for the robot to follow or robustly account for this factor in the outer-loop controller design.}

\section{METHODOLOGY}
\label{sec:main_results}
In this section, we present our proposed LipNet-based MRAC (LipNet-MRAC) approach to enforce a robot to behave as a predefined reference model. To facilitate our discussion, in Sec.~\ref{subsec:lipschitz_network}, we present a brief background on the LipNet~\cite{anil2019sorting}. In Sec.~\ref{subsec:model_reference_adaptive_network}, we derive an ideal model reference adaptation law based on the dynamics model of the robot system. In Sec.~\ref{subsec:online_learning}, we introduce an online algorithm to learn the model reference adaptation law with a LipNet when the robot dynamics are unknown, and in Sec.~\ref{subsec:stability_analysis}, we derive a Lispchitz condition that guarantees stability of the proposed LipNet-MRAC approach.

\subsection{Background on Lipschitz Networks}
\label{subsec:lipschitz_network}
In contrast to conventional feedforward networks whose Lipschitz constants are often difficult to estimate~\cite{fazlyab2019efficient}, LipNets have exact, predefined Lipschitz constants that the designer can choose freely~\cite{anil2019sorting}. Setting and knowing the Lipschitz constant is critical for guaranteeing stability of NN-based control frameworks~\cite{zhou-ijrr20,shi2019neural}. 

In this work, we consider an $M$-layer neural network $T_\theta(\xi)$ that can be expressed as follows:
\begin{equation}
\label{eqn:dnn}
    T_\theta(\xi) = W^{M}\sigma(W^{M-1}\sigma(\hdots\sigma(W^1{\xi}+b^{1}))+b^{M-1})+b^{M}\,,
\end{equation}
where $\xi$ is the input of the network, $\{W^1, W^2, ...,W^{M}\}$ are the weights matrices, $\{b^1, b^2,...,b^{M}\}$ are the bias vectors, $\theta$ denotes an augmented vector of the network weight and bias parameters, and $\sigma(\cdot)$ is the activation function.

Different from conventional networks, LipNets enforce exact Lispchitz constraints by ensuring that the input-output gradient norm is preserved by each linear and activation layer: $||J_l^Tz_l|| = ||z_l||$, where $z_l$ and $J_l$ are the input and the input-output Jacobian of layer $l$, and $||\cdot||$ is the Euclidean norm of a vector. To realize gradient norm preservation, \cite{anil2019sorting} proposes to \textit{(i)} orthonormalize the weight matrices in each linear layer such that the weight matrices have singular values of 1 exactly, and \textit{(ii)} use a gradient-preserving activation function \emph{GroupSort} that sorts the input to the hidden layer. More specifically, the \emph{GroupSort} activation function divides the input to the hidden layer into groups and sorts the values of each group in ascending or descending order. As an example, with full sort, we have $[1,2,3,4]^T = \text{GroupSort}([3,2,4,1]^T)$.

Since the \emph{GroupSort} activation function only permutates the inputs to the layer, the input-output gradient norm of the GroupSort layer is 1. By design, the overall network has a Lipschitz constant of~1. The 1-Lipschitz network can be extended to approximate a function with an arbitrary Lipschitz constant by scaling the output of the network by the desired Lipschitz constant~\cite{anil2019sorting}. In contrast to spectral normalization approaches, where the weight matrices of the network are scaled by their spectral norms, LipNets have exact Lipschitz constants, which reduces the conservatism for imposing Lipschitz constraints.

\subsection{Model Reference Adaptive Law}
\label{subsec:model_reference_adaptive_network}
In this subsection, using the representations of the robot system~\eqref{eqn:actual_system} and the reference model~\eqref{eqn:nonlinear_reference_system}, we derive the model reference adaptive law to be approximated by the LipNet.

To facilitate our discussion, we introduce the notion of system relative degree. We define $f_a\circ g_a$ as the composition $f_a(g_a(\cdot))$ of the functions $f_a$ and $g_a$, and $f_a^{i}$ as the $i$th composition of the function $f_a$ with $f_a^0(x) = f_a(x)$ and $f_a^i(x) = f_a^{i-1}(x) \circ f_a(x)$. As discussed in~\cite{zhou-ijrr20}, a nonlinear system~\eqref{eqn:actual_system} is said to have a relative degree of $r$, if $r$ is the smallest integer such that $\frac{\partial}{\partial u_{a}} h_a \circ f_a^{r-1} (f_a(x) + g_a(x)u_a(x))\neq 0$ in a neighbourhood of an operating point $(\bar{x}_{a}, \bar{u}_{a})$. Intuitively, for a discrete-time system, the relative degree $r$ defines the number of sample delays between applying an input $u_a$ to the system and seeing a corresponding change in the output~$y_a$.

By leveraging the definition of relative degree, we can relate the input $u_a$ and output $y_a$ of the robot system~\eqref{eqn:actual_system}:
\begin{align}
\label{eqn:input-output-actual}
y_{a,k+r} &= \mathcal{F}_a(x_{a,k}) + \mathcal{G}_a(x_{a,k}) u_{a,k},
\end{align}
where $\mathcal{F}_a(x_{a,k}) = h_a \circ f_a^{r}(x_{a,k})$ and $\mathcal{G}_a(x_{a,k}) = \frac{\partial}{\partial u_{a,k}} h_a \circ f_a^{r-1}(f_a(x_{a,k})+g_a(x_{a,k})u_{a,k})$. This input-output relationship allows us to predict the future output of the robot system $y_{a,k+r}$ based on the current input $u_{a,k}$ and state $x_{a,k}$. 

By assuming that the reference model is designed to have the same relative degree~$r$, we can similarly derive the input-output equation of the reference system:
\begin{align}
	y_{m,k+r} &= \mathcal{F}_m(x_{m,k}) + \mathcal{G}_m(x_{m,k}) u_{m,k},
\end{align}
where $\mathcal{F}_m(x_{m,k})$ and $\mathcal{G}_m(x_{m,k})$ are defined analogously to $\mathcal{F}_a(x_{a,k})$ and $\mathcal{G}_a(x_{a,k})$ for~\eqref{eqn:actual_system}. For a linear reference system~\eqref{eqn:linear_reference_system}, the input-output equation reduces to:
\begin{align}
    y_{m,k+r} &= \mathcal{A}_m x_{m,k} + \mathcal{B}_mu_{m,k},
\end{align}
where $\mathcal{A}_m = C_mA_m^r$ and $\mathcal{B}_m = C_mA^{r-1}_mB_m$.

Recall the architecture in Figure~\ref{fig:blockdiagram}, where the inputs to the robot system and the reference model are defined by \eqref{eqn:input_actual}. In order to make the robot system behave like the reference model, we enforce the outputs of the robot system and the reference model to be identical. In particular, by setting $y_{a,k+r} = y_{m,k+r} $ and solving for $\delta u_k$, one can show that the ideal input adjustment $\delta u_k$ for model reference adaptation is 
\begin{align}
    \delta u_k =\frac{\Fm - \Fa + \Gm\: u_k}{\Ga} - u_k.
    \label{eqn:adaptation_law}
\end{align}

When the robot dynamics is unknown, we can treat the ideal adjustment~\eqref{eqn:adaptation_law} as a nonlinear function that maps from the robot system state~$x_{a,k}$, the reference model state~$x_{m,k}$, and the input signal $u_{k}$ to the input adjustment $\delta u_k$:
\begin{align}
\delta u_k = T_\theta(x_{a,k}, x_{m,k}, u_{k}).
\end{align}

\subsection{Online Learning of the Model Reference Adaptation Law}
\label{subsec:online_learning}
In this subsection, we outline an online learning algorithm to discover the ideal adaptation law~\eqref{eqn:adaptation_law} \emph{via} a Lipschitz network when the robot dynamics~\eqref{eqn:actual_system} is not known. 

We define the performance error of the neural network as the difference between the output of the robot system~\eqref{eqn:actual_system} and the output of the reference model~\eqref{eqn:nonlinear_reference_system}: $E_{k} = y_{m,k+r} - y_{a,k+r}$. At each time step~$k$, the parameters of the Lipschitz network are updated to minimize the squared error cost function:
\begin{align}
\label{eqn:cost_function}
\mathcal{J}_{k} = \frac{1}{2} E_{k}^2 = \frac{1}{2}\left(y_{m,k+r} - y_{a,k+r} \right)^2.
\end{align}

We use the following gradient-based approach to update the network parameters, $\theta_{k+1} =\theta_k  + \Delta \theta_k$. The change in the network parameters is:
\begin{align}
\label{eqn:delta_theta}
\Delta \theta_k = -\lambda\nabla_\theta \mathcal{J}_k=\lambda H_k G_k E_k,
\end{align}
where $\lambda > 0$ is the learning rate, $G_k = \nabla_\theta T_\theta |_{x_{a,k}, x_{m,k}, u_{m,k}}$ is the gradient of the network output with respect to its parameters evaluated at $(x_{a,k}, x_{m,k}, u_{m,k})$, and $H_k = \nabla_{u_{a,k}} y_{a,k+r}$ is the input-output gradient of the robot system.
 
To realize the online adaptation law~\eqref{eqn:delta_theta}, we need to predict the system output $y_{a,k+r}$ and estimate the input-to-output gradient $\nabla_{u_{a,k}} y_{a,k+r}$. Similar to~\cite{jordan1992forward,zhou-ecc19}, we can simultaneously learn a forward model for the robot system to estimate $y_{a,k+r}$ and $\nabla_{u_{a,k}} y_{a,k+r}$ (see \eqref{eqn:input-output-actual}):
\begin{remark}[Forward Model Learning~\cite{zhou-ecc19}]
At time $k$, one can construct a paired dataset with inputs $\{y_{a,p-r}, u_{a,p-r}\}$ and outputs $\{y_{a,p}\}$ based on the latest $N$ time steps $p=\{k-N,...,k\}$ and use standard supervised learning to train a forward model (e.g., a BLR model) as a local approximator of~\eqref{eqn:input-output-actual}. The model can be then used to estimate $y_{a,k+r}$ and $\nabla_{u_{a,k}} y_{a,k+r}$ by setting the input to $(x_{a,k}, u_{a,k})$.
\end{remark}

We note that inaccuracies in the forward dynamics model could, in general, lower the adaptation performance but will not jeopardize the stability of the adapted system. As will be shown in Sec.~\ref{subsec:stability_analysis}, the stability of the proposed LipNet-MRAC approach is guaranteed if the Lipschitz constant of the LipNet satisfies a small-gain-type condition.
 
In the case where a prediction model is not available, one could still apply the proposed algorithm for model reference adaptation but with a sample delay of $r$ steps, which is typically a small integer for robot systems such as quadrotors. For a linear system, $\nabla_{u_{a,k}} y_{a,k+r}$ is a constant that can be factored into the learning rate $\lambda$ as a tuning parameter, and its estimation is not required. 

\subsection{Stability Analysis}
\label{subsec:stability_analysis}
In this subsection, we provide stability guarantees of the system including the model reference adaptation law by exploiting the Lipschitz property of the learning module. In the stability analysis, we make the following assumptions:
\begin{itemize}
\item[(A1)] The state of the robot system can be bounded by $||x_a||_{l}\le  \gamma||u_a||_{l} + \beta$, where $\gamma$ and $\beta$ are positive constant scalars, $||\cdot||_{l}$ denotes the $l_2$ signal norm, and the variables without the subscripts $k$ denote the corresponding signals.
\item[(A2)] The input adjustment computed by the adaptation module satisfies $\delta u_k = 0$ for $(x_{a,k}, x_{m,k}, u_k) = (0, 0, 0)$. 
\item[(A3)] The state of the reference system $x_{m}$ is bounded (i.e., $||x_m||_{l}< \infty$).
\end{itemize}

Assumption~(A1) holds for finite-gain $l_2$ stable systems and is common assumption in small-gain-type theorems, which are the basis of the proof presented below. The scalar $\gamma$ is an upper bound on the input-to-state gain of the robot system, and the scalar $\beta$ is a constant value associated with the initial state of the robot system. As shown in the proof below, $\beta$ affects the upper bound on the state of the system but does not impact the stability of the adapted system. Assumption~(A2) is true for any robot and reference systems satisfying $\mathcal{F}_a(0)=0$ and $\mathcal{F}_m(0)=0$. This condition is not restrictive and can be practically enforced by removing the bias vectors from the LipNet architecture. Assumption~(A3) can be satisfied by a proper choice of stable reference system.

\begin{theorem}
Consider the proposed LipNet-MRAC approach shown in Figure~\ref{fig:blockdiagram} (grey box). Under assumptions (A1)-(A3), the dynamics of the adapted system from $u$ to $x_a$ is finite-gain~$l_2$ stable if $L <1/\gamma$, where $L$ is the Lipschitz constant of the LipNet, which we are free to choose, and $\gamma$ is an upper bound on the input-to-state gain of the robot system.
\end{theorem}

\begin{proof}
By assumption (A1), the state of the robot system can be bounded as follows: $||x_a||_{l}\le  \gamma||u_a||_{l} + \beta = \gamma||u + \delta u||_{l} + \beta \le \gamma||u||_{l} +\gamma|| \delta u||_{l} + \beta$. 
Moreover, by assumption (A2) and the Lipschitz property of the LipNet,  
at any instance, the input adjustment computed by the LipNet can be bounded as $||\delta u_k|| \le  L||\xi_k||$, where $L$ is the Lipschitz constant of the network, and $\xi = [x_a^T, x_m^T, u]^T$ denotes the network input. It follows that $||\delta u||_{l}=\left(\sum_{k=0}^\infty ||\delta u_k||^2\right)^{1/2} \le \left(\sum_{k=0}^\infty L^2 ||\xi_k||^2 \right)^{1/2} = L||\xi||_{l}\le L||x_a||_{l} + L||x_m||_{l} + L||u||_{l}$. 
Using the upper bound on $||\delta u||_{l}$, we obtain $||x_a||_{l}\le  \gamma(1+ L)||u||_{l} +\gamma L||x_a||_{l} + \gamma L||x_m||_{l} + \beta$. It can be shown that, if $L <1/\gamma$ is satisfied, the state of the robot system can be bounded by $||x_a||_{l}\le  \big( \gamma(1+L)||u||_{l} + \gamma L||x_m||_{l} + \beta \big)/(1-\gamma L)$. Since, by assumption (A3), $||x_m||_{l}$ is bounded, the dynamics of the adapted system from $u$ to $x_a$ is finite-gain $l_2$ stable~\cite{khalil2002nonlinear} (cf. Figure~\ref{fig:blockdiagram_detailed}).
\end{proof}

Theorem~1 provides an upper bound on the Lipschitz constant of the adaptive network module to guarantee stability. This Lipschitz condition can be enforced via the architecture design of the LipNet (Sec.~\ref{subsec:lipschitz_network}). To enforce the Lipschitz condition, we require an estimate of the upper bound of the system gain~$\gamma$, which can be estimated from system input-output data~\cite{van2007data,james1995numerical} or chosen conservatively based on our prior knowledge of the system. Overestimating~$\gamma$ will lead to a smaller, more conservative Lipschitz constant for the LipNet, but the overall adapted system will remain stable.

\section{SIMULATION EXAMPLE}
\label{sec:simulation_results}
In this section, we present a numerical example to illustrate the proposed LipNet-MRAC approach.

We consider the following system: 
\begin{equation}
\label{eqn:simulation_system_a}
\begin{aligned}
x_{a,k+1} &= \begin{bmatrix}
1 & T\\-T&1-T
\end{bmatrix} x_{a,k} + d(x_{a,k}) + \begin{bmatrix}
0 \\ 0.6T
\end{bmatrix} u_{k},\\ y_{a,k} &= \begin{bmatrix}
1 & 1
\end{bmatrix} x_{a,k},
\end{aligned}
\end{equation}
where $d( x_{a,k}) = 0.1T \begin{bmatrix}
x_{a,k,1}\sin(x_{a,k,1}),  & 0
\end{bmatrix}^T$
with $T=0.01$ and $x_{a,k,1}$ being the first element of $x_{a,k}$. The gain of system~\eqref{eqn:simulation_system_a} has an upper bound of $\gamma = 1.12$. \edittext{The system~\eqref{eqn:simulation_system_a} has a relative degree of 1.} The reference model is
\begin{equation}
\label{eqn:simulation_reference_system}
\begin{aligned}
x_{m,k+1} &=  \begin{bmatrix}
1 & T \\ -0.25T & 1 -T
\end{bmatrix} x_{m,k} +  \begin{bmatrix}
0\\T \end{bmatrix}u_{m,k},\\
y_{m,k} &= \begin{bmatrix}
0.25 & 0.25
\end{bmatrix} x_{m,k}\,.
\end{aligned}
\end{equation}
\edittext{The reference system~\eqref{eqn:simulation_reference_system} also has a relative degree of 1.}

\begin{figure}
\centering
\vspace{0.5em}
\includegraphics[trim=0cm 0cm 0cm 0cm, width=\columnwidth]{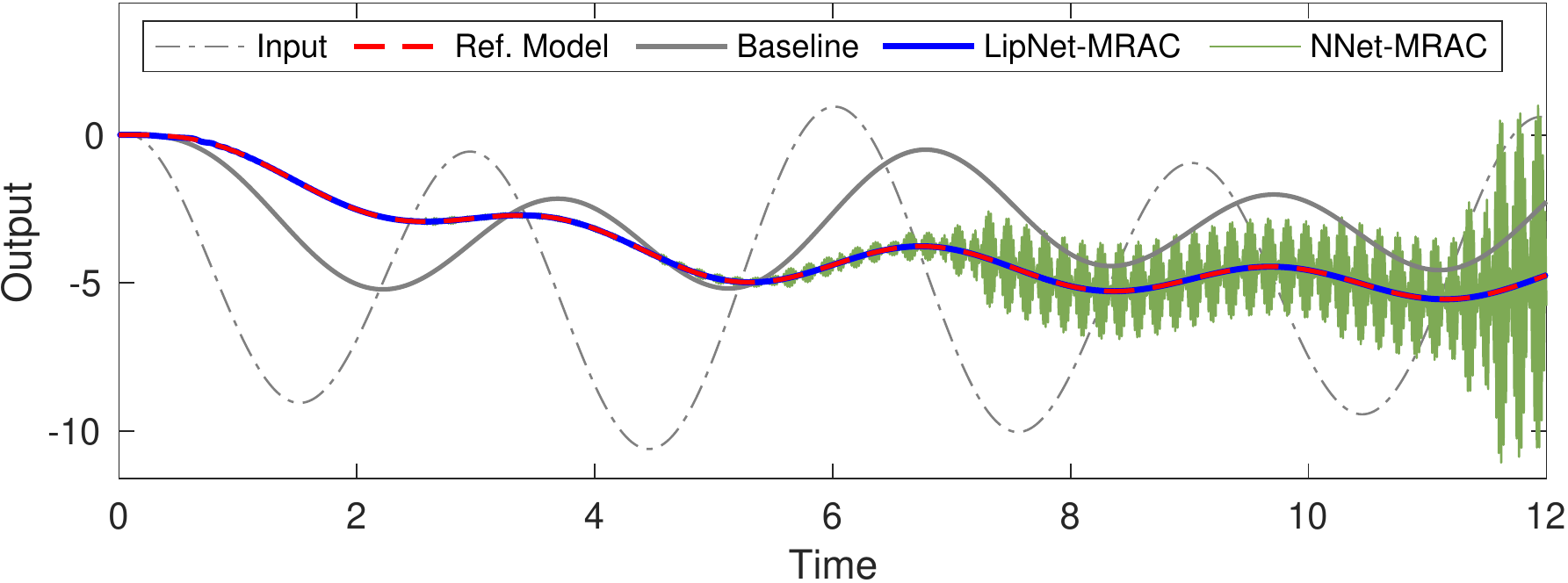}
\caption{The proposed LipNet-MRAC approach effectively enforces the input-output response of system~\eqref{eqn:simulation_system_a} to behave as the the reference model~\eqref{eqn:simulation_reference_system} (blue and red). The baseline response without adaptation is shown in grey. In contrast, with the conventional NNet-MRAC (green), closed-loop stability is not guaranteed. This simulation corresponds to one test input trajectory $u_k = \sin\left(\frac{2\pi }{5}kT\right) + 5 \cos\left(\frac{2\pi }{3}  kT\right) - 5$ with $T=0.01$. The learning rate is set to 33 for both approaches.}
\label{fig:simulation_trajctory}
\end{figure}

\begin{figure}
\centering
\includegraphics[trim=0cm 0cm 0cm 0cm, width=\columnwidth]{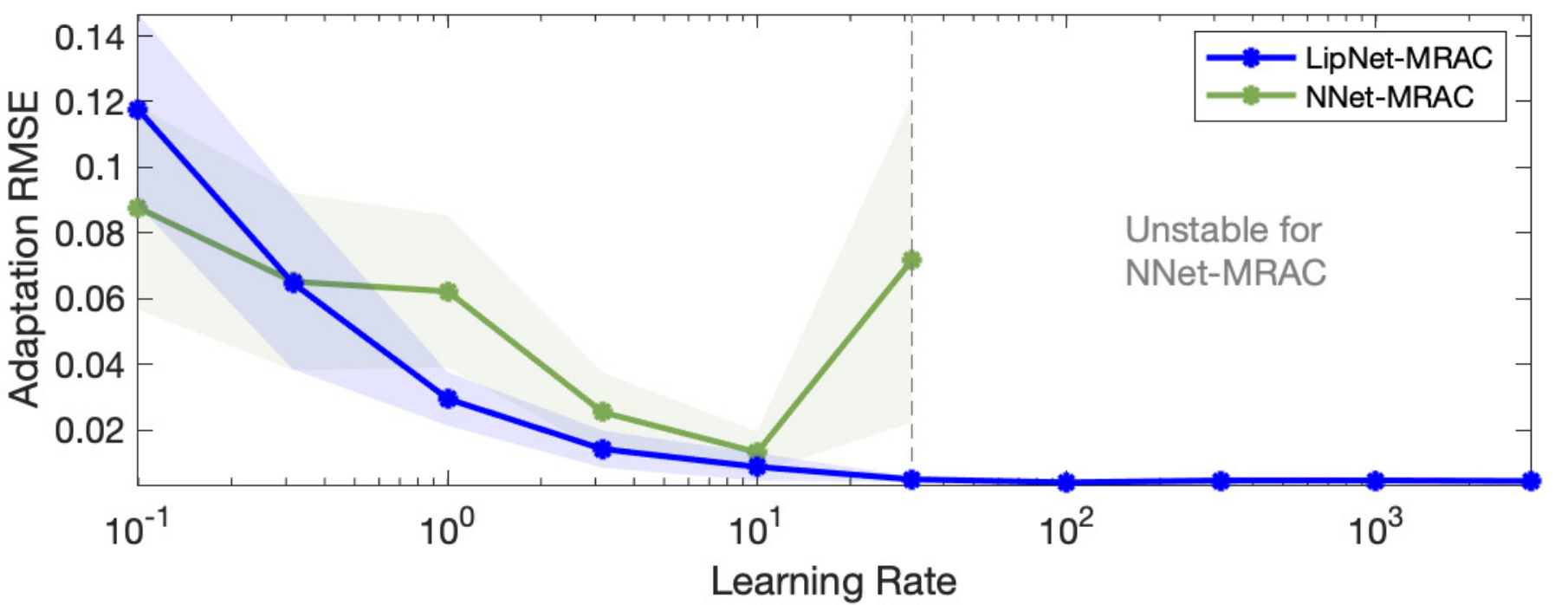}
\caption{The performance of the proposed LipNet-MRAC approach and the NNet-MRAC on one test trajectory when different learning rates are used. Using the LipNet-MRAC approach, we can always guarantee stability, and the adaptation performance asymptotically approaches a lower bound as the learning rate increases. However, with the NNet-MRAC approach, there is an ideal learning rate that needs to be carefully chosen, which can be challenging to find when we do not know the robot dynamics a-priori. The solid lines and the shades show the means and one standard deviations for ten trials with different initial network parameters.}
\label{fig:simulation_learning_rate_cases}
\reducefiggap
\end{figure}

Our goal is to design an adaptive module such that the system output~\eqref{eqn:simulation_system_a} tracks the output of the reference model~\eqref{eqn:simulation_reference_system}. In the discussion below, we first illustrate the efficacy of using the proposed adaptive LipNet-MRAC approach to make a nonlinear system~\eqref{eqn:simulation_system_a} behave as a linear reference system~\eqref{eqn:simulation_reference_system} without knowing the dynamics model of the nonlinear system a-priori. We then show the benefit of using the proposed LipNet-MRAC approach by comparing it to a learning-based MRAC approach with a conventional feedforward network architecture (NNet). Both the LipNet and the NNet have a depth and width of 3 and 20. The LipNet has FullSort hidden layers and orthogonalized linear layers~\cite{anil2019sorting}, while the NNet has $\tanh$ hidden layers and standard linear layers. The same adaptation scheme~(Sec.~\ref{subsec:online_learning}) is applied to update the network parameters. The initial parameters of the networks are randomly sampled from the standard normal distribution. We compare the two approaches over ten randomly-initialized trials. To satisfy Theorem~1 with the proposed LipNet-MRAC approach, the Lipschitz constant of the LipNet is set to $1/\gamma = 0.89$ to guarantee stability.

Figure~\ref{fig:simulation_trajctory} shows the response of the system~\eqref{eqn:simulation_system_a} when using \textit{(i)} the proposed LipNet-MRAC approach, and \textit{(ii)} a learning-based MRAC approach with a conventional feedforward network architecture (abbreviated as NNet). By comparing the baseline response of system~\eqref{eqn:simulation_system_a} (grey line) and the response of system~\eqref{eqn:simulation_system_a} with the adaptive LipNet (blue line), we can see that the proposed approach effectively enforces the dynamics of system~\eqref{eqn:simulation_system_a} to behave as the reference model (red dashed line) as desired. With the conventional NNet (green line), stability is not guaranteed.

Figure~\ref{fig:simulation_learning_rate_cases} compares the adaptation error when different learning rates are used with NNet and the proposed LipNet. For each learning rate, the plot shows the mean and standard deviation of the root-mean-square (RMS) error over ten randomly-initialized trials. NNet has one ideal learning rate for which the mismatch between the system and the reference model is the lowest. 
Searching for this ideal learning rate requires trial-and-error and the system can be destabilized for higher values. 
For the LipNet-MRAC approach, the stability of the system is not jeopardized, regardless of the chosen learning rate. \edittext{Higher learning rates generally allow for faster adaptation to any mismatches between the reference model and the robot system.}
As a result, the adaptation RMSE for the LipNet-MRAC case asymptotically approaches an ideal value as the learning rate increases. 
By encoding the Lipschitz condition (Theorem~1) in the LipNet design, we can safely increase the learning rate for faster adaptation while guaranteeing stability a-priori despite network parameter initialization. 

\section{EXPERIMENTAL RESULTS}
\label{sec:experimental_results}
We demonstrate the proposed LipNet-MRAC approach through flying inverted pendulum experiments. A video of the quadrotor experiments presented in this section can be found here: \url{http://tiny.cc/lipnet-pendulum} 
\begin{figure}
    \centering
    \vspace{0.5em}
    \includegraphics[width=0.4\columnwidth]{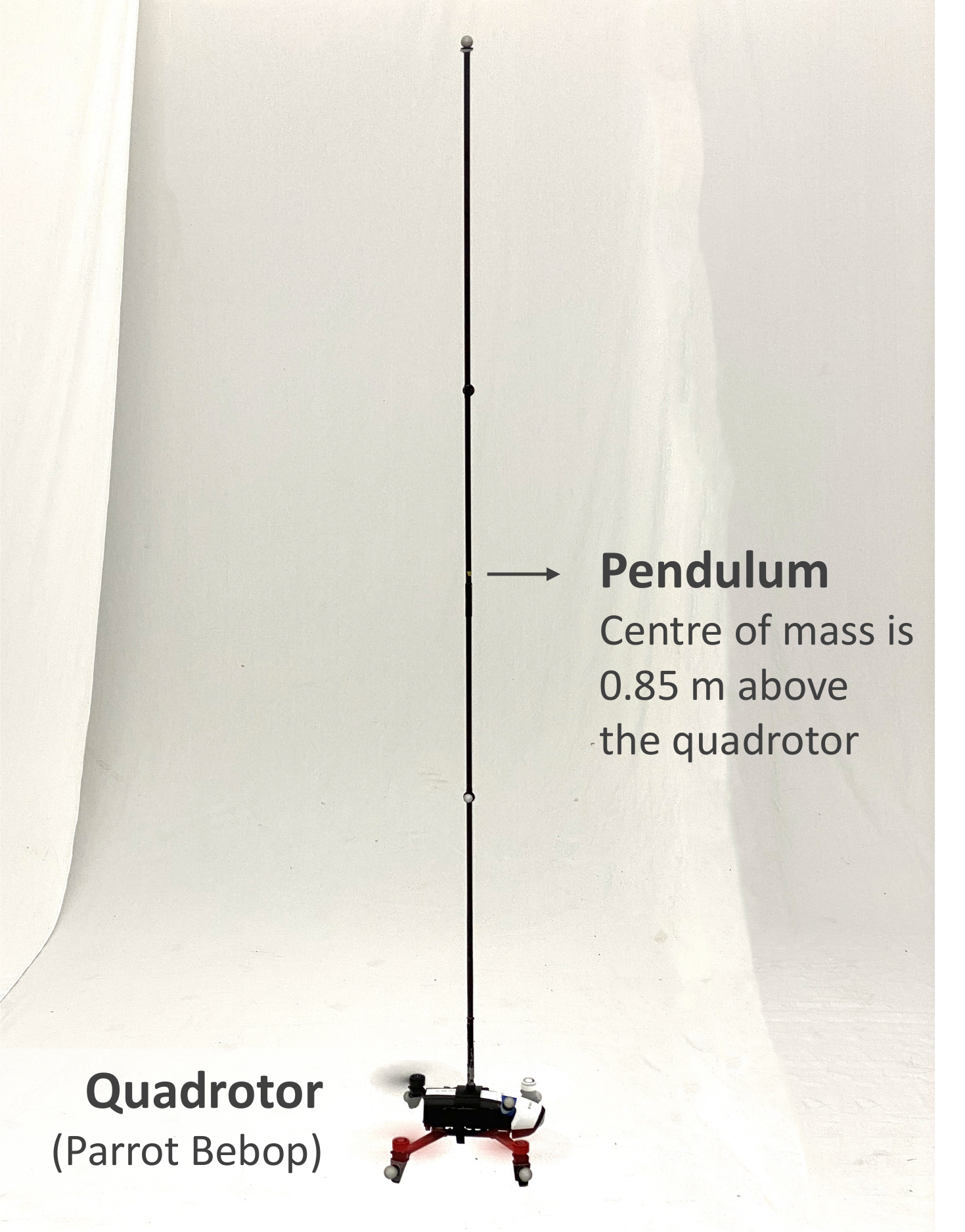}
    \vspace{-0.5em}
    \caption{We demonstrate our proposed approach using a flying inverted pendulum, where a quadrotor (Parrot Bebop) balances a pendulum while hovering at a fixed point or tracking a trajectory. }
    \label{fig:experiment_setup}
    \reducefiggap
\end{figure}

\subsection{Experimental Setup}
The goal of the experiment is to stabilize an inverted pendulum on a quadrotor vehicle (the Bebop) while hovering and tracking a trajectory in the $xy$-plane. The state of the quadrotor consists of the translational positions of its centre of mass (COM) $(x,y,z)$, the translational velocities $(\dot{x}, \dot{y}, \dot{z})$, the roll-pitch-yaw Euler angles $(\phi, \theta, \psi)$, and the angular velocities $(\omega_x, \omega_y, \omega_z)$. We model the pendulum as a point mass~\cite{hehn2011flying}. To capture the dynamics of the pendulum, we define four additional states $(r, s, \dot{r}, \dot{s})$, which correspond to the positions and velocities of the COM of the pendulum relative to the positions and velocities of the COM of the quadrotor along the $x$ and the $y$ axes---the pendulum is balanced in the upright position when both $r$ and $s$ are zero. An illustration of the experimental setup is shown in Figure~\ref{fig:experiment_setup}.

By assuming that the quadrotor is stabilized at a constant height (i.e., $\dot{z} = 0$), we can represent the translational dynamics of flying inverted pendulum system in the form below:
\begin{align}
\label{eqn:quad_pend_dyn_1}
   \mathbf{x}_{a,k+1} = f_a(\mathbf{x}_{a,k}) + g_a(\mathbf{x}_{a,k})\mathbf{a}_{a,k},
\end{align}
where the state $\mathbf{x}_a = (x_a,\dot{x}_a,r_a,\dot{r}_a,y_a,\dot{y}_a,s_a,\dot{s}_a)$ is an augmentation of the pertinent states of the quadrotor and the pendulum, and the input $\mathbf{a}_a = (a_{a,x}, a_{a,y})$ is the actual acceleration of the quadrotor~\cite{hehn2011flying}. 

\begin{figure}
    \centering
    \vspace{0.5em}
    \includegraphics[trim=0cm 0cm 0cm 0cm, width=\columnwidth]{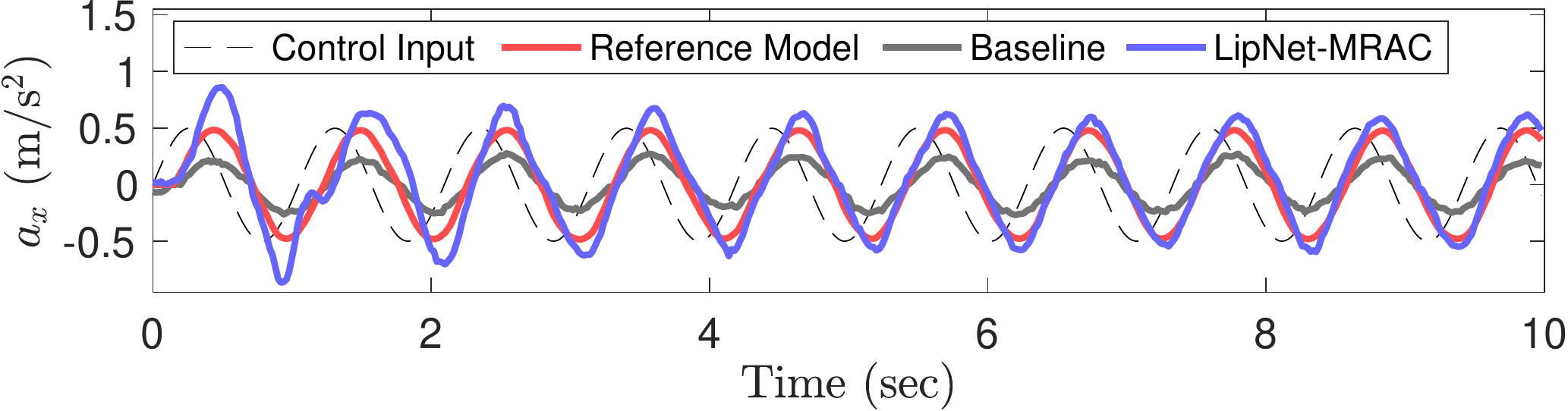}
    \caption{Given the \edittext{control input} (dashed line), the proposed LipNet-MRAC approach allows the actual acceleration of the quadrotor (blue) to closely follow the output of the reference model (red). The response of the baseline system without adaptation is shown in grey. A similar result is observed for the acceleration tracking in the $y$-direction.
    }
    \label{fig:trajectory_tracking}
        \reducefiggap
\end{figure}

\begin{figure}
    \centering
    \includegraphics[trim=0.6cm 0cm 1.55cm 0.5cm, width=0.85\columnwidth]{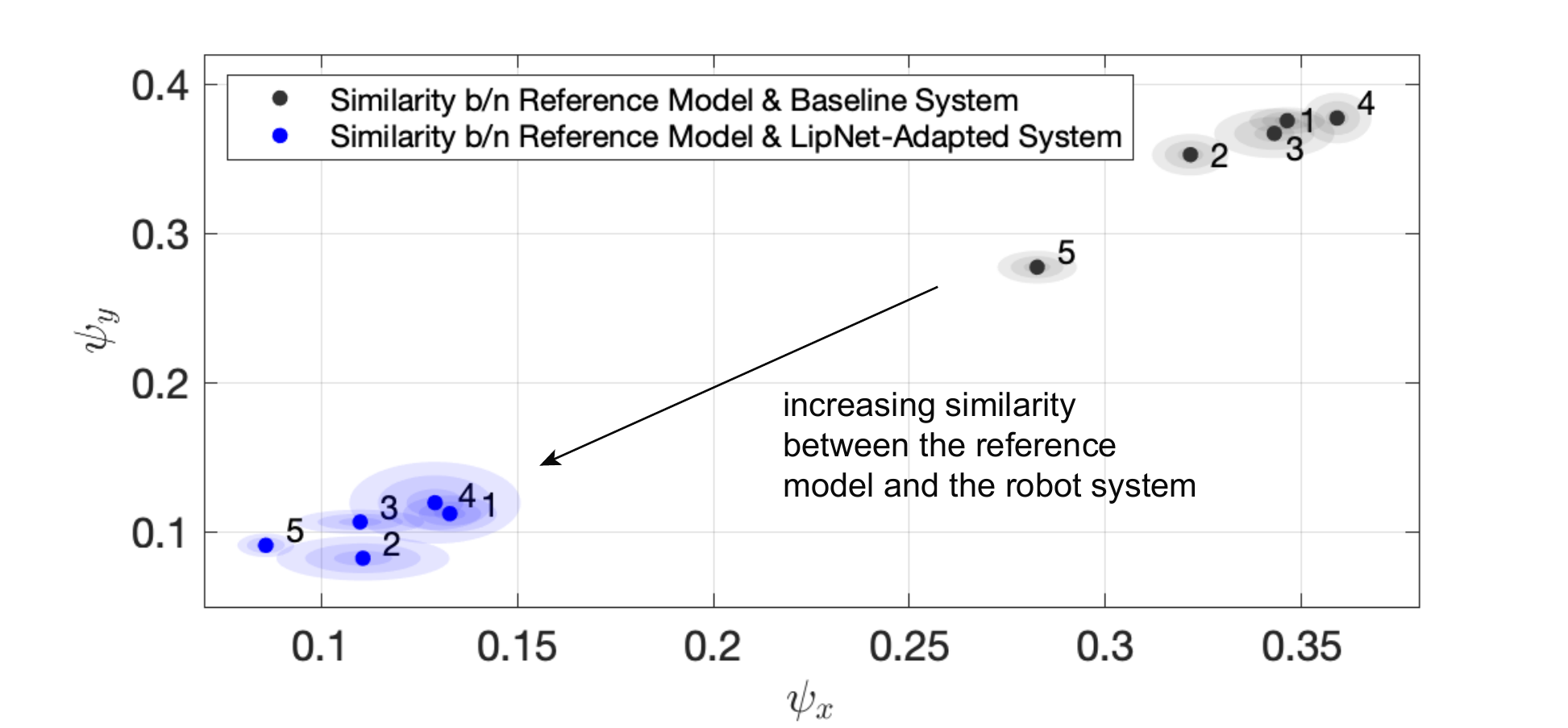}
    \caption{The proposed LipNet-MRAC approach can effectively enforce the robot system to behave as the randomly selected reference models. The plot shows a comparison of the system similarity between the five reference models and \emph{(i)} the system with the baseline controller (grey), and \emph{(ii)} the system with the proposed LipNet-MRAC module (blue). Smaller values of $\psi_x$ and $\psi_y$ indicate a higher system similarity between the reference model and the system in terms of the $\nu$-gap metric~\cite{zhou1998essentials}; 
    the dots and shaded areas in the plot correspond to the means and 3$\sigma$ error bounds of the $\nu$-gap estimates obtained based on the algorithm outlined in~\cite{sorocky-icra20}. 
    }
    \label{fig:similarity_estimates}
        \reducefiggap
\end{figure}

To design a stabilizing controller for the quadrotor-pendulum system in~\eqref{eqn:quad_pend_dyn_1}, one could first design a controller to compute the required acceleration of the quadrotor to stabilize the quadrotor-pendulum dynamics~\eqref{eqn:quad_pend_dyn_1} and then use an inner-loop attitude controller to ensure that the desired acceleration is achieved~\cite{hehn2011flying}. However, in our experiments, we do not have access to the attitude control of the off-the-shelf quadrotor. We instead apply the proposed LipNet-MRAC approach outside of the attitude control loop to make the acceleration dynamics of the quadrotor behave as a predefined reference model:
\begin{align}
\label{eqn:acceleration_ref_model}
\mathbf{a}_{m,k+1} = A_m \mathbf{a}_{m,k} + B_m \mathbf{u}_{m,k},
\end{align}
where $\mathbf{u}_m = [u_{m,x}, u_{m,y}]^T$ is the acceleration command. The reference model is then incorporated into the overall quadrotor-pendulum dynamics model as an extended system:
\begin{align}
\label{eqn:quad_pend_dyn_extended}
    \boldsymbol{\xi}_{a,k+1} =\begin{bmatrix}f_a(\mathbf{x}_{a,k}) + g_a(\mathbf{x}_{a,k})\mathbf{a}_{a,k}\\ A_m \mathbf{a}_{a,k} \end{bmatrix}+ \begin{bmatrix}
        \mathbf{0}\\B_m 
    \end{bmatrix}\mathbf{u}_{k},
\end{align}
where $\boldsymbol{\xi}_{a}=[\mathbf{x}_{a}^T, \mathbf{a}_{a}^T]^T$ is the state of the extended system, and the input $\mathbf{u}$ is the acceleration command of the quadrotor. \edittext{Note that, following~\cite{zhou-ijrr20}, we can estimate the relative degree of an uncertain robot system from simple experiments. In our case, the robot system and the reference model have a relative degree of~1.
}

Given the model in~\eqref{eqn:quad_pend_dyn_extended}, we can use a standard model-based controller to design a feedback control law for stabilizing the quadrotor-pendulum system. In this work, we use a standard linear quadratic regulator (LQR) of the form $\mathbf{u}_{k} = \mathbf{K}\boldsymbol{\widetilde{\xi}}_{a,k}$, where $\mathbf{K}$ is the controller gain designed based on~\eqref{eqn:quad_pend_dyn_extended}, $\boldsymbol{\widetilde{\xi}}_{a,k}$ is the error in the extended state relative to a desired state, which is constant for stabilization tasks and time-varying for tracking tasks. Note that, to compensate for the input-output delay present in the quadrotor system, we introduced a lead compensator with a forward prediction in the closed-loop system. Similar to the LQR controller, the parameters of the lead compensator are determined based on~\eqref{eqn:quad_pend_dyn_extended}. 

To ensure that the acceleration dynamics of the quadrotor follow the reference model~\eqref{eqn:acceleration_ref_model},  we assume decoupled quadrotor acceleration dynamics in the $x$- and $y$-directions and use the proposed LipNet-MRAC approach outlined in Sec.~\ref{sec:main_results}. In the experiments, the adaptive LipNets have  depths of 3 and widths of 20. By observing the input-output responses of the baseline quadrotor attitude controller on a set of sinusoidal trajectories, the quadrotor system gain $\gamma$ is estimated to be 0.68. Based on Theorem~1, we conservatively set the Lipschitz constant of the LipNets to 0.8. To train the LipNet online, we simultaneously fit a local BLR model to approximate the forward acceleration dynamics. The parameters of the LipNet are updated to minimize the cost~\eqref{eqn:cost_function} with $\lambda = 0.8$. 

With the proposed LipNet-MRAC, the acceleration command from the LQR controller is adjusted by the adaptive LipNet (Figure~\ref{fig:blockdiagram_detailed}) and the overall acceleration command sent to the quadrotor is $\mathbf{u}_{a,k}=\mathbf{u}_{k}+\delta\mathbf{u}_{k}$, where $\delta\mathbf{u}_{k}$ is the adjustment computed by the LipNet. Using the Euler parameterization of the attitude angles, the acceleration command $\mathbf{u}=[u_{x}, u_{y}]^T$ is converted to the attitude commands based on the following transformations: $\theta_c = \arctan \left(u_{x}/g\right)$ and $\phi_c = \arctan\left(-u_{y}/\sqrt{u_{x}^2+g^2}\right)$, where $\theta_c$ and $\phi_c$ are the commanded pitch and roll angles, and $g$ is the acceleration due to gravity. The attitude commands are sent  to the Bebop quadrotor onboard controller at a rate of 50~Hz. 

Our experiments consist of \textit{(i)} verifying efficacy of the proposed LipNet-MRAC for making the quadrotor system behave as a reference model and \textit{(ii)} demonstrating the LipNet-MRAC in closed-loop control for a flying inverted pendulum.
\begin{figure}
    \centering
    \vspace{1em}
    \begin{subfigure}{.5\textwidth}
    \centering
    \includegraphics[trim=0cm 0cm 0cm 0cm, width=\columnwidth]{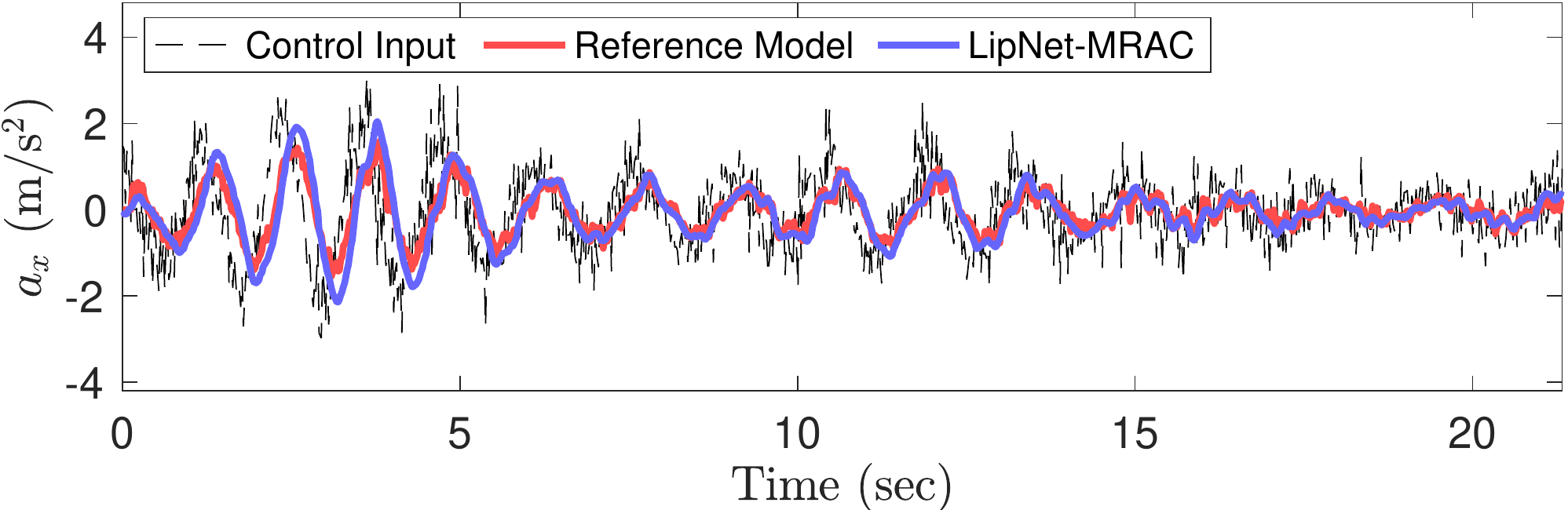}
    \caption{Performance validation of the underlying LipNet-MRAC module. The actual acceleration of the quadrotor (blue) closely follows the output acceleration of the reference model (red). Similar result is observed for $a_y$. The control input signal (black) is generated by the high-level LQR controller. }
    \label{fig:pendulum_tracking_acc_stab}
    \end{subfigure}\\[1em]
    \begin{subfigure}{.5\textwidth}
    \centering
    \includegraphics[trim=0cm 0cm 0cm 0cm, width=\columnwidth]{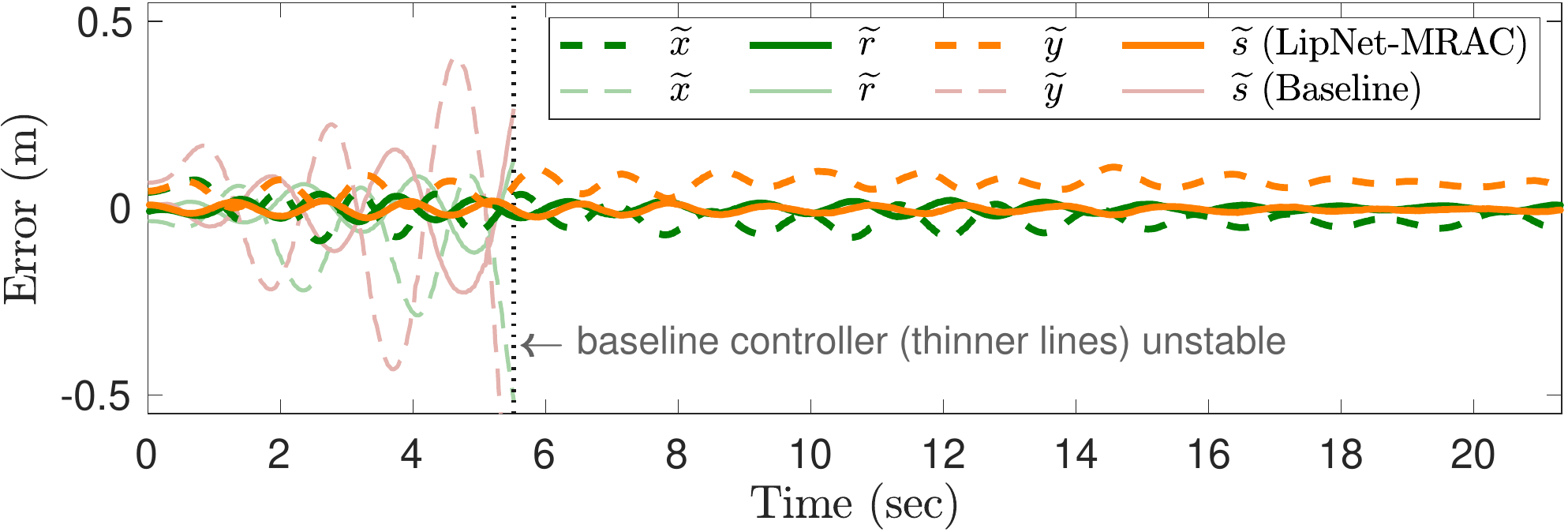}
    \caption{The error of the quadrotor positions $\widetilde{x}$ and $\widetilde{y}$ (dashed lines) and the pendulum relative positions $\widetilde{r}$ and $\widetilde{s}$ (solid lines) when the quadrotor balances the pendulum while hovering at a fixed position. The proposed LipNet-MRAC approach (thicker lines) enables the pendulum to be balanced in the upright position despite dynamics uncertainties. Without the LipNet-MRAC, due to the model-reality gap, the baseline controller alone (thinner lines) cannot stabilize the flying inverted pendulum system.}
    \label{fig:pendulum_tracking_poserr_stab}
    \end{subfigure}
    \caption{Quadrotor balancing a pendulum while hovering.}
    \reducefiggap
\end{figure}

\subsection{LipNet-MRAC for Predictable Acceleration Dynamics}
We first show that the proposed LipNet-MRAC can make the acceleration dynamics of the Bebop quadrotor behave as different predefined reference models. For simplicity of the outer-loop controller design, we choose linear reference acceleration models of the following form:
\begin{align}
\label{eqn:quad_pend_dyn}
   \mathbf{a}_{m,k+1} = \begin{bmatrix} \beta_{mx} & 0\\0&\beta_{my}\end{bmatrix}\mathbf{a}_{m,k} + \begin{bmatrix} \alpha_{mx} & 0\\0&\alpha_{my}\end{bmatrix}\mathbf{u}_{m,k},
\end{align}
where $\tau_m = (\alpha_{mx}, \beta_{mx}, \alpha_{my}, \beta_{my})$ are model parameters.

To illustrate the idea of our proposed approach, we first set the reference model parameters to $\tau_m = (0.35, 0.65, 0.35, 0.65)$. Figure~\ref{fig:trajectory_tracking} shows the quadrotor system response with the baseline controller, and with the LipNet-MRAC on one test trajectory. As can be seen from the plot, the adaptive LipNet brings the acceleration response of the quadrotor system close to the given reference model. 

To further demonstrate the efficacy of the LipNet-MRAC approach, we randomly sample five sets of model parameters $\tau_m$ and apply the LipNet without any fine tuning of the learning algorithm parameters. To formally evaluate the performance of the reference model adaptation approach, we use the $\nu$-gap metric from robust control~\cite{zhou1998essentials} to measure the `distance' between the reference model and the quadrotor system response with and without LipNet adaptation. Intuitively, two dynamical systems that are close in term of the $\nu$-gap can be stabilized by the same controller. Figure~\ref{fig:similarity_estimates} shows the estimated $\nu$-gap metric using experimental data from the quadrotor and the iterative algorithm outlined in \cite{sorocky-icra20}. A smaller $\nu$-gap value indicates a higher similarity between the reference model and the quadrotor system. The plot shows that the LipNet-MRAC approach can reliably make the quadrotor system behave close to the five reference models. In the next subsection, we apply LipNet-MRAC to the flying inverted pendulum problem.

\begin{figure}
    \centering
    \vspace{1em}
    \begin{subfigure}{.5\textwidth}
    \centering
    \includegraphics[trim=0cm 0cm 0cm 0cm, width=\columnwidth]{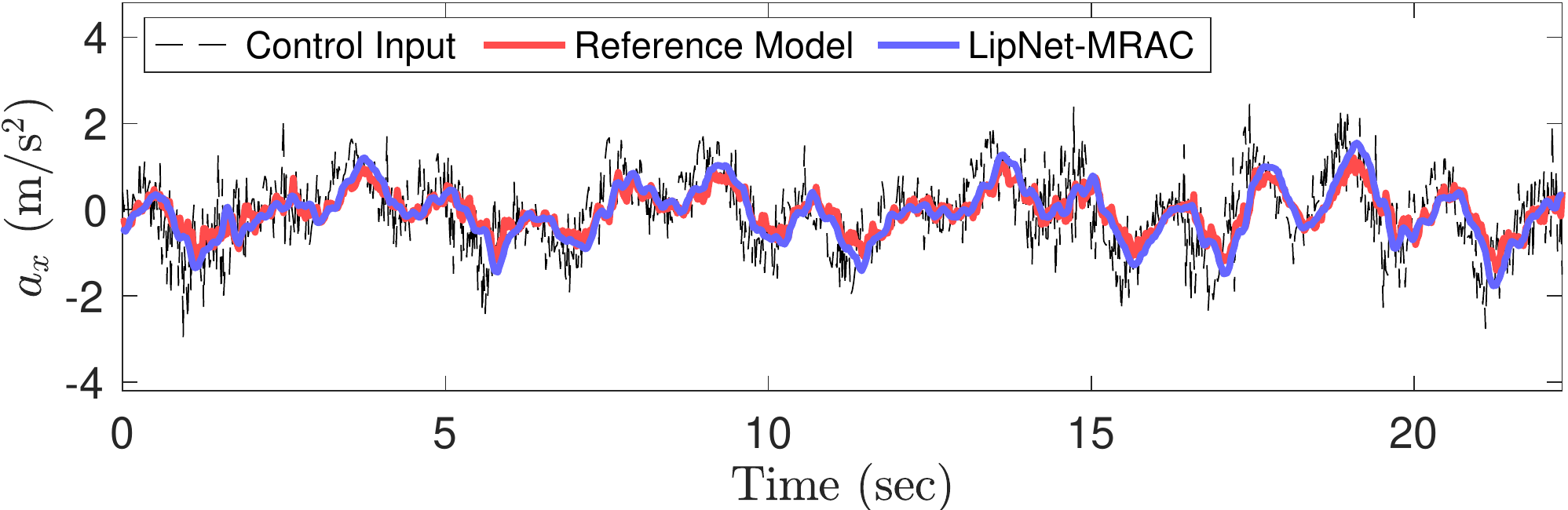}
    \caption{Performance validation of the underlying LipNet-MRAC module. The actual acceleration of the quadrotor (blue) closely follows the output acceleration of the reference model (red). Similar result is observed for $a_y$. The control input signal (black) is generated by the high-level LQR controller.}
    \label{fig:pendulum_tracking_acc_traj}
    \end{subfigure}\\[1em]
    \begin{subfigure}{.5\textwidth}
    \centering
    \includegraphics[trim=0cm 0cm 0cm 0cm, width=\columnwidth]{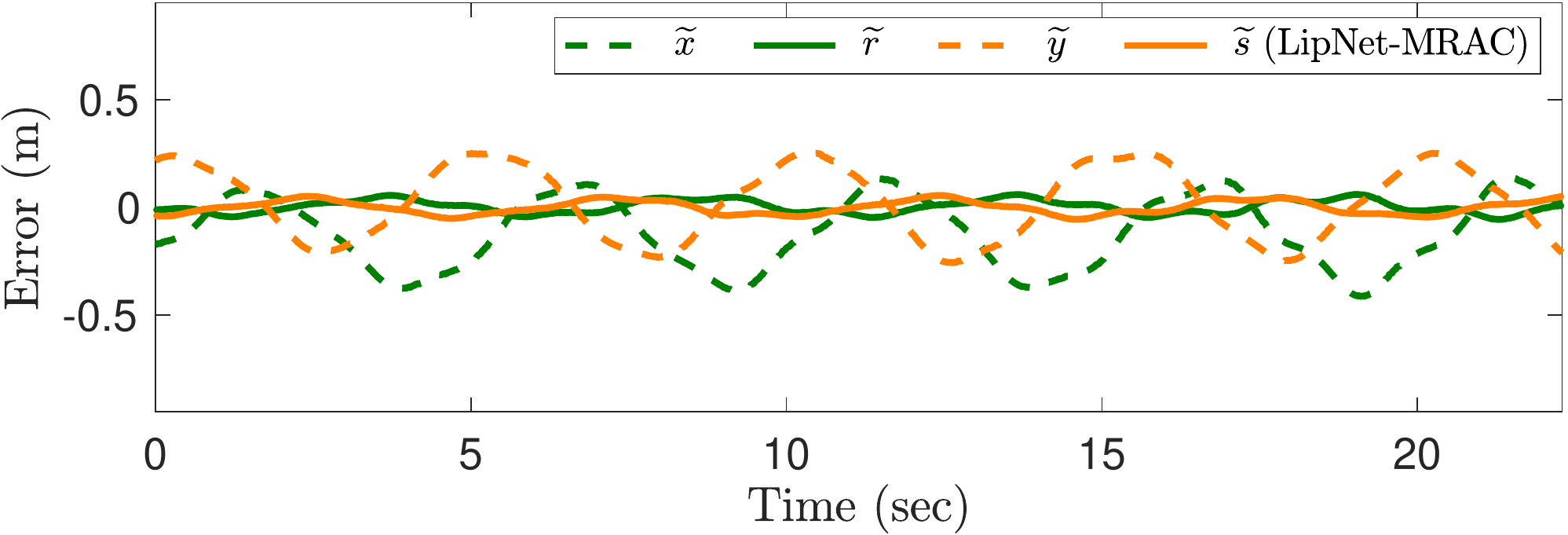}
    \caption{The LipNet-MRAC approach allows the quadrotor to balance the pendulum while tracking a circular trajectory with a radius of 0.25~m and angular velocity of 1.25~rad/sec. The RMS error in the quadrotor positions (dashed lines) and the pendulum positions (solid lines) are 0.27 m and 0.04 m, respectively.
    }
    \label{fig:pendulum_tracking_poserr_traj}
    \end{subfigure}
    \caption{Quadrotor balancing a pendulum while tracking a trajectory.}
    \reducefiggap
\end{figure}

\subsection{Inverted Pendulum on a Quadrotor Experiments}
An LQR stabilization controller is designed based on the dynamics in~\eqref{eqn:quad_pend_dyn_extended}, where the reference acceleration model has the form of~\eqref{eqn:quad_pend_dyn}. In the controller design process, we expect that the quadrotor system behaves as the reference model; we do not need to explicitly model the acceleration dynamics of the quadrotor system or modify the default attitude controller onboard of the quadrotor platform. Our experiments encompass the following tests: \textit{(i)} pendulum stabilization, \textit{(ii)} pendulum stabilization with wind and tap disturbances, and \textit{(iii)} pendulum stabilization while tracking circular trajectories.

We first show results for the case when the quadrotor is commanded to hover at a fixed point while balancing the pendulum. Figure~\ref{fig:pendulum_tracking_acc_stab} shows the acceleration response of the quadrotor system. It can be seen that, as desired with the LipNet-MRAC, the actual acceleration of the quadrotor follows the output of the reference model. As compared to the baseline system of the quadrotor, the acceleration reference model has an input-to-output gain closer to unity, which facilitates the outer-loop LQR controller design. Figure~\ref{fig:pendulum_tracking_poserr_stab} shows the resulting errors in the pendulum and quadrotor system. Given the predictable behaviour of the acceleration dynamics, we see that the outer-loop pendulum controller can successfully balance the pendulum while keeping the quadrotor position error close to zero. The RMS error in the quadrotor and pendulum positions are 0.08~m and 0.02~m, respectively. On the contrary, if we use the baseline controller alone, there is a model-reality gap and the overall system is not stable (lighter lines in Figure~\ref{fig:pendulum_tracking_poserr_stab}).  As we demonstrate in the supplementary video, the proposed LipNet-MRAC-based controller design is even able to maintain the pendulum in the upright position when wind disturbances are applied to the quadrotor or a gentle force is applied to the pendulum.

Next, we show the case when the quadrotor is commanded to track a circular trajectory of radius 0.25~m and angular frequency 1.25~rad/sec while balancing a pendulum. Figure~\ref{fig:pendulum_tracking_acc_traj} shows the acceleration response of the quadrotor system, which closely tracks the output of the reference model despite the sharp changes in the input signal. Figure~\ref{fig:pendulum_tracking_poserr_traj} shows the position errors of the pendulum and the quadrotor. The quadrotor is able to track the circular trajectory while keeping the pendulum balanced. The RMS error in the quadrotor and pendulum positions are 0.27~m and 0.04~m, respectively.  In the supplementary video, we show that the quadrotor can successfully track circular trajectories with angular frequencies up to 2.09~rad/sec, while keeping the pendulum balanced.

\section{CONCLUSIONS}
\label{sec:conclusions}
In this paper, we presented a neural model reference adaptive approach (LipNet-MRAC) to make nonlinear systems with possibly unknown dynamics behave as a predefined reference model. By leveraging the representative power of DNNs, the proposed approach can be applied to a larger class of nonlinear systems than other approaches in the literature. Moreover, we derive a certifying Lipschitz condition that guarantees the stability of the overall adaptive LipNet framework. We applied the proposed approach to a flying inverted pendulum. Our experiments show that the proposed approach is able to make the dynamics of an unknown black-box quadrotor system behave in a predictable manner, which facilitates the outer-loop pendulum stabilization controller synthesis. By complementing a standard controller with the proposed LipNet-MRAC, we successfully stabilized an inverted pendulum with an off-the-shelf quadrotor platform whose dynamics are not known a-priori. In future work, we would like to analyze the impact of the LipNet adaptation errors on the performance of the outer-loop model-based control design.

\bibliographystyle{IEEEtran}
\bibliography{IEEEabrv,references}

\addtolength{\textheight}{-12cm}   

\vfill

\end{document}